\RequirePackage{fix-cm}
\documentclass[twocolumn]{svjour3}          
\smartqed  
\usepackage{romannum}
\usepackage{subfig}
\usepackage{nameref}
\usepackage[table]{xcolor}
\usepackage{graphicx}
\usepackage{array}
\usepackage{hyperref}
\hypersetup{colorlinks,linktocpage,plainpages=false, bookmarksopen = true}
\usepackage{multirow}
\usepackage{amsmath}
\usepackage{amssymb}
\usepackage{wrapfig}
\usepackage{soul}

\usepackage[ruled,linesnumbered]{algorithm2e}
\begin{document}
\pagenumbering{arabic}
\sloppy

\title{A Novel Intrinsic Measure of Data Separability 
}

\titlerunning{A Novel Intrinsic Measure of Data Separability}        

\author{Shuyue Guan         \and
        Murray Loew 
}


\institute{Shuyue Guan \at
              ORCID: \href{https://orcid.org/0000-0002-3779-9368}{0000-0002-3779-9368} \\
              \email{frankshuyueguan@gwu.edu}           
           \and
           Murray Loew \at
              Corresponding author \\
              \email{loew@gwu.edu} \\
              \\
           Department of Biomedical Engineering, \\
           George Washington University, Washington DC, USA
}

\date{Received: date / Accepted: date}

\maketitle

\begin{abstract}
In machine learning, the performance of a classifier depends on both the classifier model and the separability/complexity of datasets. To quantitatively measure the separability of datasets, we create an \textit{intrinsic measure} -- the Distance-based Separability Index (DSI), which is independent of the classifier model. We consider the situation in which different classes of data are mixed in the same distribution to be the most difficult for classifiers to separate. We then formally show that the DSI can indicate whether the distributions of datasets are identical for any dimensionality. And we verify the DSI to be an effective separability measure by comparing to several state-of-the-art separability/complexity measures using synthetic and real datasets. Having demonstrated the DSI's ability to compare distributions of samples, we also discuss some of its other promising applications, such as measuring the performance of generative adversarial networks (GANs) and evaluating the results of clustering methods.
\end{abstract}
\textbf{Declarations} Not applicable.

\section{Introduction}

Data and models are the two main foundations of machine learning and deep learning. Models learn knowledge (patterns) from datasets. An example is that the convolutional neural network (CNN) classifier learns how to recognize images from different classes. There are two aspects in which we examine the learning process: capability of the classifier and the separability of dataset. Separability is an intrinsic characteristic of a dataset \cite{fernandez_data_2018} to describe how data points belonging to different classes mix with each other. The learning outcomes are highly dependent on the two aspects. For a specific model, the learning capability is fixed, so that the training process depends on the training data. As reported by Chiyuan~et~al.~\cite{1}, the time to convergence for the same training loss on random labels on the CIFAR-10 dataset was greater than when training on true labels. It is not surprising that the performance of a given model varies between different training datasets depending on their separability. For example, in a two-class problem, if the scattering area for each class has no overlap, one straight line (or hyperplane) can completely separate the data points (Figure~\ref{fig:1}a). For the distribution shown in Figure~\ref{fig:1}b, however, a single straight line cannot separate the data points successfully, but a combination of many lines can.

\begin{figure}[h]
    \centerline{\includegraphics[width=0.3\textwidth]{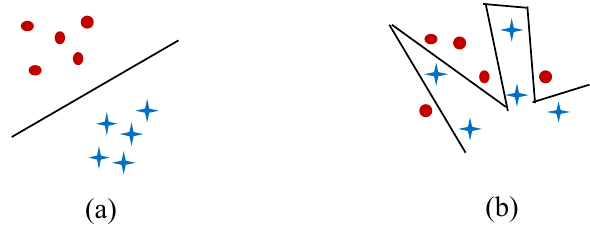}}
    \caption{Different separability of two datasets}
    \label{fig:1}
\end{figure}

In machine learning, for a given classifier, it is more difficult to train on some datasets than on others. The difficulty of training on a less-separable dataset is made evident by the requirement for greater learning times (\textit{e.g.}, number of epochs for deep learning) to reach the same accuracy (or loss) value and/or to obtain a lower accuracy (or higher loss), compared with the more-separable dataset. The training difficulty, however, also depends on the model employed. Hence, an intrinsic measure of separability that is independent of the model is required. In other words, it is very useful to be able to measure the separability of a dataset without using a classifier model. In summary, the separability of a dataset can be characterized in three ways:
\begin{enumerate}
    \item To describe how data points belonging to different classes mix with each other.
    \item To determine the number of [hyper-] planes/linear-dividers needed to separate different-class data points.
    \item To gain insights into the training of a specific classifier with regard to time-cost and final accuracy.
\end{enumerate}
Our proposed method is based on the first way. The second and third ways depend on classifiers; we used the third way to verify our method.

\subsection{Related work}
A review of the literature indicates that there have been substantially fewer studies on data separability \textit{per se} than on classifiers. The Fisher discriminant ratio (FDR) \cite{2} measures the separability of data using the mean and standard deviation (SD) of each class. It has been used in many studies, but it fails in some cases (\textit{e.g.}, as Figure~\ref{fig:4}(e) shows, Class 1 data points are scattered around Class 2 data points in a circle; their FDR $\approx0$). A more general issue than that of data separability is data complexity, which measures not only the relationship between classes but also the data distribution in feature space. Ho~and~Basu~\cite{3} conducted a groundbreaking review of data complexity measures. They reported measures for classification difficulty, including those associated with the geometrical complexity of class boundaries. Recently, Lorena~et~al.~\cite{4} summarized existing methods for the measurement of classification complexity. In the survey, most complexity measures have been grouped in six categories: \textit{feature-based}, \textit{linearity}, \textit{neighborhood}, \textit{network}, \textit{dimensionality}, and \textit{class imbalance} measures (Table~\ref{table:1}). For example, the FDR is a \textit{feature-based} measure, and the geometric separability index (GSI) proposed by Thornton~\cite{5} is considered a \textit{neighborhood} measure. Other ungrouped measures discussed in Lorena’s paper have similar characteristics to the grouped measures or may have large time cost. Each of these methods has possible drawbacks. In particular, the features extracted from data for the five categories of \textit{feature-based} measures may not accurately describe some key characteristics of the data; some \textit{linearity} measures depend on the classifier used, such as support-vector machines (SVMs); \textit{neighborhood} measures may show only local information; some \textit{network} measures may also be affected by local relationships between classes depending on the computational methods employed; \textit{dimensionality} measures are not strongly related to classification complexity; and, \textit{class imbalance} measures do not take the distribution of data into account.

\begin{table}[h]
    \caption{Complexity measures reported by Lorena~et~al.~\cite{4}} 
    \label{table:1}
    \centering
    \resizebox{0.48\textwidth}{!}{
    \begin{tabular}{l p{0.3\textwidth} c}
    \hline\noalign{\smallskip}
    \bfseries Category & \bfseries Name & \bfseries Code \\
    \noalign{\smallskip}\hline\noalign{\smallskip}
    \multirow{8}{*}{Feature-based} & Maximum Fisher’s discriminant ratio & F1\\
             & Directional vector maximum Fisher’s discriminant ratio & F1v\\
             & Volume of overlapping region  & F2\\
             & Maximum individual feature efficiency & F3\\
             & Collective feature efficiency  & F4\\
    \noalign{\smallskip}\hline\noalign{\smallskip}
    \multirow{5}{*}{Linearity} & Sum of the error distance for linear programming & L1\\
             & Error rate of the linear classifier & L2\\
             & Non-linearity of the linear classifier  & L3\\
    \noalign{\smallskip}\hline\noalign{\smallskip}
        \multirow{8}{*}{Neighborhood} & Fraction of borderline points & N1\\
             & Ratio of intra/extra class NN distance & N2\\
             & Error rate of the NN classifier  & N3\\
             & Non-linearity of the NN classifier & N4\\
             & Fraction of hyperspheres covering the data & T1\\
             & Local set average cardinality & LSC\\
    \noalign{\smallskip}\hline\noalign{\smallskip}
        \multirow{3}{*}{Network} & Density & Density\\
             & Clustering coefficient & ClsCoef\\
             & Hubs & Hubs\\
    \noalign{\smallskip}\hline\noalign{\smallskip}
        \multirow{6}{*}{Dimensionality} & Average number of features per dimension & T2\\
             & Average number of PCA dimensions per point & T3\\
             & Ratio of the PCA dimension to the original dimension & T4\\
    \noalign{\smallskip}\hline\noalign{\smallskip}
        \multirow{2}{*}{Class imbalance} & Entropy of class proportions & C1\\
             & Imbalance ratio & C2\\
    \noalign{\smallskip}\hline
    \end{tabular}}
\end{table}

\begin{figure*}[t]
    \centering
        \includegraphics[width=0.8\textwidth]{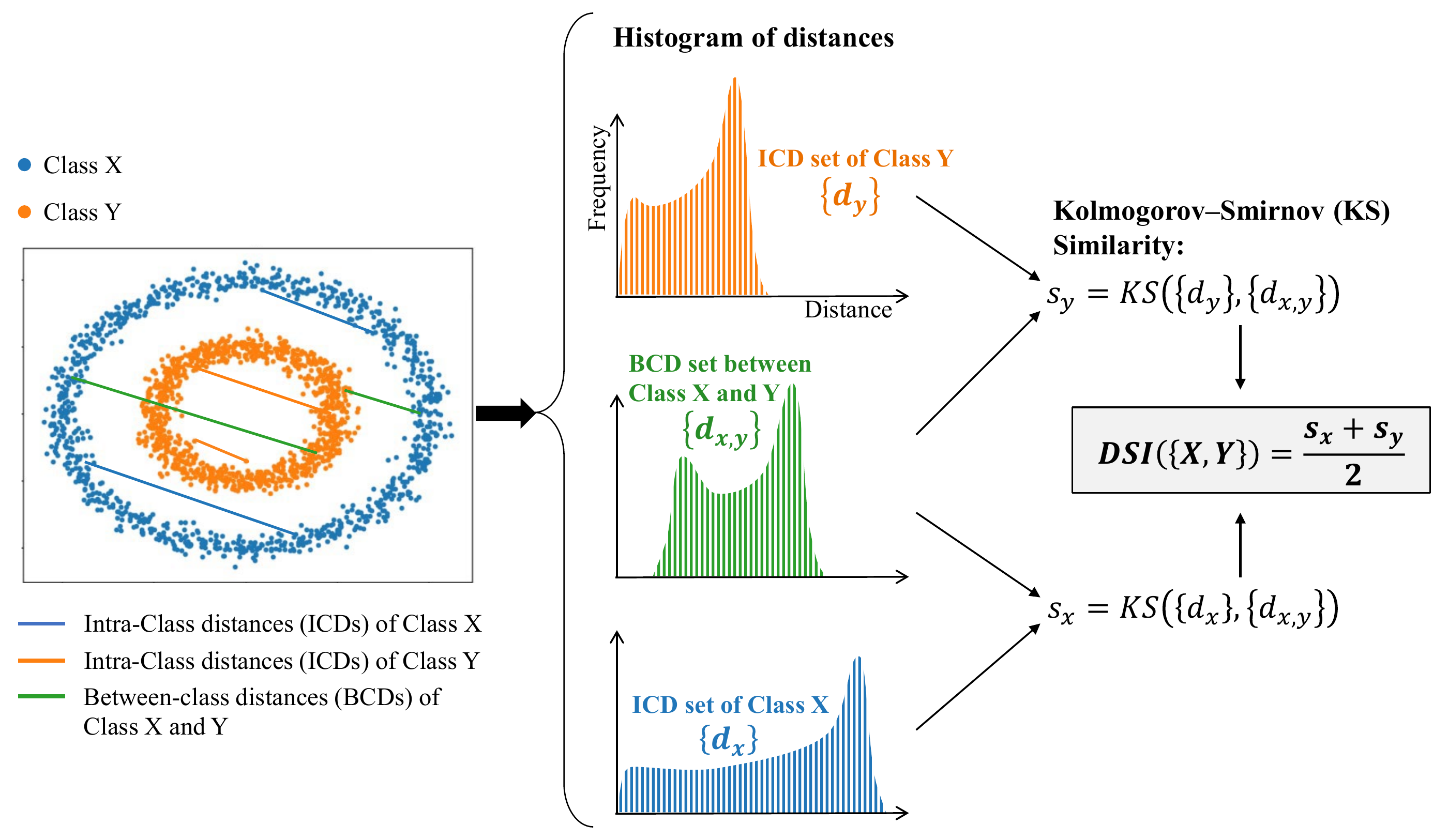}
    \caption{An example of two-class dataset in 2-D shows the definition and computation of the DSI. Details about the ICD and BCD sets are in Section~\ref{icd_bcd}, and Section~\ref{KS} contains more details about computation of the DSI. The proof that DSI can measure the separability of dataset is shown in Section~\ref{core_thm}.}
    \label{fig:DSI_abs}
\end{figure*}

In this paper, we create a novel separability measure for multi-class datasets and verify it by comparing with other separability/complexity measures using synthetic and real (CIFAR-10/100) datasets. Since several previous studies \cite{5,6,7,8,9} have used the term separability index (SI), we refer to our measure as the distance-based separability index (DSI). The DSI measure is similar in some respects to the \textit{network} measures because it represents the universal relations between the data points. Especially, we have formally shown that the DSI can indicate whether the distributions of datasets are identical for any dimensionality. Figure~\ref{fig:DSI_abs} shows the definition and computation of the DSI by an example of a two-class dataset in 2-D. In general, the DSI has a wide applicability and is not limited to simply understanding the data; for example, it can also be applied to measure generative adversarial network (GAN) performance \cite{guan2021novel}, evaluate clustering results \cite{guan2020cluster}, anomaly detection \cite{2021anomaly}, the selection of classifiers \cite{10,11,12,13}, and features for classification \cite{14,15}.

The novelty of this study is the examination of the 1) distributions of datasets via 2) the distributions of \underline{distances} between data points in datasets; the proved Theorem connects the two kinds of distributions; that is the gist of DSI. To the best of our knowledge, none of the existing studies uses the same methods. 

\section{Methodological Development for DSI}
\label{Methodological development}

\begin{figure}[h]
    \centering
        \includegraphics{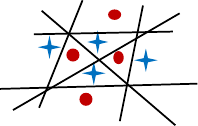}
    \caption{Two-class dataset with maximum entropy}
    \label{fig:2}
\end{figure}

In a two-class dataset, we consider that the most difficult situation to separate the data is when the two classes of data are scattered and mixed together with the same distribution. In this situation, the proportion of each class in every small region is equal, and the system has maximum entropy. In extreme cases, to obtain 100\% classification accuracy, the classifier must separate each data point into an individual region (Figure~\ref{fig:2}).

Therefore, one possible definition of data separability is the inverse of a system's entropy. To calculate entropy, the space could be randomly divided into many small regions. Then, the proportions of each class in every small region can be considered as their occurrence probabilities. The system's entropy can be derived from those probabilities \cite{Shannon}. In high-dimensional space (\textit{e.g.}, image data), however, the number of small regions grows exponentially. For example, the space for $32\times 32$ pixels 8-bit RGB images has 3,072 dimensions. If each dimension (ranging from 0 to 255, integer) is divided into 32 intervals, the total number of small regions is $32^{3072}\approx 6.62\times 10^{4623}$. It is thus impossible to compute the system's entropy and analyze data separability in this way.

Alternatively, we can define the data separability as the similarity of data distributions. If a dataset contains two classes $X$ and $Y$ with the same distribution (distributions have the same shape, position and support, \textit{i.e.} the same probability density function) and have sufficient data points to fill the region, this dataset reaches the maximum entropy because within any small regions, the occurrence probabilities of the two classes data are equal (50\%). It is also the most difficult situation for separation of the dataset.

Here, we proposed a new method -- distance-based separability index (DSI) to measure the similarity of data distributions. DSI is used to analyze how two classes of data are mixed together, as a substitute for entropy.

\subsection{Intra-class and between-class distance sets}
\label{icd_bcd}
Before introducing the DSI, we introduce the intra-class distance (ICD) and the between-class distance (BCD) set that are used for computation of the DSI. The ``set'' in this paper means the \textit{multiset} that allows \underline{duplicate} elements (distance values, in our cases), and the $|A|$ of a ``set'' A is the number of its elements. 

Suppose $X$ and $Y$ have $N_x$  and $N_y$ data points, respectively, we can define:

\begin{definition} \label{def:1}
The intra-class distance (ICD) set $\{d_x\}$ is a set of distances between any two points in the same class $(X)$, as: $\{d_x\}=\{ \|x_i-x_j\|_2 | x_i,x_j\in X;x_i\neq x_j\}$.
\end{definition}

\begin{corollary} \label{cor:1}
Given $|X|=N_x$, then $|\{d_x\}|=\frac{1}{2}N_x(N_x-~1)$.
\end{corollary}

\begin{definition} \label{def:2}
The between-class distance (BCD) set $\{ d_{x,y}\}$ is the set of distances between any two points from different classes $(X \, and \, Y)$, as $\{ d_{x,y} \}=\{ \|x_i-~y_j\|_2 \, |\, x_i\in X;y_j\in Y \}$.
\end{definition}

\begin{corollary} \label{cor:2}
Given $|X|=N_x,|Y|=N_y$, then $|\{d_{x,y} \}|=N_x N_y$.
\end{corollary}

\begin{remark}
The metric for all distances is \textit{Euclidean} $(l^2\,\text{norm})$ in this paper. In Section~\ref{metrics}, we compare the Euclidean distance with some other distance metrics including City-block, Chebyshev, Correlation, Cosine, and Mahalanobis, and we showed that the DSI based on Euclidean distance has the best sensitivity to complexity, and thus we selected it.
\end{remark}

\subsection{Definition and computation of the DSI}
\label{KS}

We firstly introduce the computation of the DSI for a dataset contains only two classes $X$ and $Y$:
\begin{enumerate}
    \item First, the ICD sets of $X$ and $Y$: $\{d_x\},\{d_y\}$ and the BCD set: $\{d_{x,y}\}$ are computed by their definitions (Defs.~\ref{def:1}~and~\ref{def:2}). 

    \item Second, the similarities between the ICD and BCD sets are then computed using the the Kolmogorov–Smirnov (KS)~\cite{21} distance\footnote{In experiments, we used the \texttt{scipy.stats.ks\_2samp} from the SciPy package in Python to compute the KS distance. \url{https://docs.scipy.org/doc/scipy/reference/generated/scipy.stats.ks_2samp.html}}:
    \[
    s_x=KS(\{d_x\},\{d_{x,y}\}),\ \text{and}\ s_y=KS(\{d_y\},\{d_{x,y}\}).
    \]
    We explain the reasons to choose the KS distance in Section~\ref{measures}.
    
    \item Finally, the DSI is the average of the two KS distances:
    \[
    DSI(\{X,Y\})=\frac{(s_x+s_y )}{2}.
    \]
    
\end{enumerate}

\begin{remark}
We do not use the weighted average because once the distributions of the ICD and BCD sets can be well characterized, the sizes of $X$ and $Y$ will not affect the KS distances $s_x$ and $s_y$.
\end{remark}

For a multi-class dataset, the DSI can be computed by one-\textit{versus}-others; specifically, for a $n$-class dataset, the process to obtain its DSI is:
\begin{enumerate}
    \item Compute $n$ ICD sets for each class: $\{d_{C_i}\};\; i=1,2,\cdots,n$, and compute $n$ BCD sets for each class. For the $i$-th class of data $C_i$, the BCD set is the set of distances between any two points in $C_i$ and $\overline{C_i}$ (other classes, not $C_i$): $\{d_{C_i,\overline{C_i}}\}$.
    \item Compute the $n$ KS distances between ICD and BCD sets for each class: 
    \[
    s_i=KS(\{d_{C_i }\},\{d_{C_i,\overline{C_i}}\}).
    \]
    \item Calculate the average of the $n$ KS distances; the DSI of this dataset is:
    \[
    DSI(\{C_i \})=\frac{\sum s_i}{n}.
    \]
\end{enumerate}

Therefore, DSI is (defined as) the mean value of KS distances between the ICD and BCD sets for each class of data in a dataset.
\begin{remark}
$DSI\in (0,1)$. A small DSI (low separability) means that the ICD and BCD sets are very similar. In this case, by Theorem \ref{thm:1}, the distributions of datasets are similar too. Hence, these datasets are difficult to separate.
\end{remark}

\subsection{Theorem: DSI and similarity of data distributions}
\label{core_thm}

Then, the Theorem \ref{thm:1} shows how the ICD and BCD sets are related to the distributions of the two-class data; it demonstrates the core value of this study.

\begin{theorem} \label{thm:1}
When $|X|\ \text{and} \ |Y|\to \infty$, if and only if the two classes $X$ and $Y$ have the same distribution, the distributions of the ICD and BCD sets are identical.
\end{theorem}

The full proof of Theorem \ref{thm:1} is in Section \ref{prof}. Here we provide an informal explanation: 
\begin{quote}
Data points in $X$ and $Y$ having the same distribution can be considered to have been sampled from one distribution $Z$. Hence, both ICDs of $X$ and $Y$, and BCDs between $X$ and $Y$ are actually ICDs of $Z$. Consequently, the distributions of ICDs and BCDs are identical. In other words, that the distributions of the ICD and BCD sets are identical indicates all labels are assigned randomly and thus, the dataset has the least separability. 
\end{quote}

According to this theorem, that the distributions of the ICD and BCD sets are identical indicates that the dataset has maximum entropy because $X$ and $Y$ have the same distribution. Thus, as we discussed before, the dataset has the lowest separability. And in this situation, the dataset's DSI $\approx 0$ by its definition.

The time costs for computing the ICD and BCD sets increase linearly with the number of dimensions and quadratically with the number of data points. It is much better than computing the dataset’s entropy by dividing the space into many small regions. Our experiments (in Section~\ref{subset_cifar}) show that the time costs could be greatly reduced using a small random subset of the entire dataset without significantly affecting the results (Figure~\ref{fig:5}). And in practice, the computation of DSI can be sped-up considerably by using tensor-based matrix multiplications on a GPU (\textit{e.g.}, it takes about 2.4 seconds for 4000 images from CIFAR-10 running on a GTX 1080 Ti graphics card) because the main time-cost is the computation of distances.

\subsection{Proof of Theorem 1}
\label{prof}

Consider two classes $X$ and $Y$ that have the same distribution (distributions have the same shape, position, and support, \textit{i.e.}, the same probability density function) and have sufficient data points ($|X|\ \text{and} \ |Y|\to \infty$) to fill their support domains. Suppose $X$ and $Y$ have $N_x$ and $N_y$ data points, and assume the sampling density ratio is $\frac{N_y}{N_x} =\alpha$. Before providing the proof of Theorem \ref{thm:1}, we firstly prove Lemma \ref{lma:1}, which will be used later.

\begin{remark}
The condition of most relevant equations in the proof is that the $N_x$ and $N_y$ are approaching infinity in the limit.
\end{remark}

\begin{lemma} \label{lma:1}
If and only if two classes $X$ and $Y$ have the same distribution covering region $\Omega$ and $\frac{N_y}{N_x} =\alpha$, for any sub-region $\Delta \subseteq \Omega$, with $X$ and $Y$ having $n_{xi},n_{yi}$ points, $\frac{n_{yi}}{n_{xi}} =\alpha$ holds.
\end{lemma}

\begin{proof}
Assume the distributions of $X$ and $Y$ are $f(x)$ and $g(y)$. In the union region of $X$ and $Y$, arbitrarily take one tiny cell (region) $\Delta_i$ with $n_{xi}=\Delta_if(x_i)N_x, n_{yi}=\Delta_ig(y_j)N_y; x_i=y_j$. Then,
\[
\frac{n_{yi}}{n_{xi}}=\frac{\Delta_ig(x_i)N_y}{\Delta_if(x_i)N_x}=\alpha \frac{g(x_i)}{f(x_i)}
\]
Therefore:
\[
\alpha \frac{g(x_i)}{f(x_i)} =\alpha \Leftrightarrow \frac{g(x_i)}{f(x_i)}=1 \Leftrightarrow \forall x_i:g(x_i)=f(x_i)
\]  \hfill $\square$
\end{proof}

\subsubsection{Sufficient condition}
\textbf{Sufficient condition of Theorem \ref{thm:1}.} \textit{If the two classes $X$ and $Y$ with the same distribution and have sufficient data points ($|X|\ \text{and} \ |Y|\to \infty$), then the distributions of the ICD and BCD sets are nearly identical.}

\begin{figure}[h]
    \centerline{\includegraphics{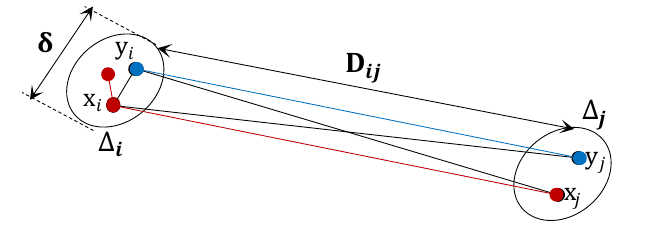}}
    \caption{Two non-overlapping small cells}
    \label{fig:3}
\end{figure}

\begin{proof}
Within the area, select two tiny non-overlapping cells (regions) $\Delta_i$  and $\Delta_j$ (Figure~\ref{fig:3}). Since $X$ and $Y$ have the same distribution but in general different densities, the number of points in the two cells $n_{xi},n_{yi};n_{xj},n_{yj}$ fulfills:
\[
\frac{n_{yi}}{n_{xi}} =\frac{n_{yj}}{n_{xj}} =\alpha
\]
The scale of cells is $\delta$, the ICDs and BCDs of $X$ and $Y$ data points in cell $\Delta_i$ are approximately $\delta$ because the cell is sufficiently small. By the Definition~\ref{def:1}~and~\ref{def:2}:
\[
d_{x_i}\approx d_{x_i,y_i}\approx \delta;\quad x_i,y_i\in \Delta_i
\]
Similarly, the ICDs and BCDs of $X$ and $Y$ data points between cells $\Delta_i$  and $\Delta_j$ are approximately the distance between the two cells $D_{ij}$:
\[
d_{x_{ij}}\approx d_{x_i,y_j}\approx d_{y_i,x_j}\approx D_{ij};\; x_i,y_i\in \Delta_i;\, x_j,y_j\in \Delta_j
\]
First, divide the whole distribution region into many non-overlapping cells. Arbitrarily select two cells $\Delta_i$  and $\Delta_j$ to examine the ICD set for $X$ and the BCD set for $X$ and $Y$. By Corollaries \ref{cor:1} and \ref{cor:2}:

\romannum{1} ) The ICD set for $X$ has two distances: $\delta$ and $D_{ij}$, and their numbers are:
\[
d_{x_i}\approx \delta;\; x_i\in \Delta_i:\; |\{d_{x_i}\}|=\frac{1}{2}n_{xi}(n_{xi}-1)
\]
\[
d_{x_{ij}}\approx D_{ij};\; x_i\in \Delta_i;x_j\in \Delta_j:\; |\{d_{x_{ij}}\}|=n_{xi}n_{xj}
\]
\romannum{2} ) The BCD set for $X$ and $Y$ also has two distances: $\delta$ and $D_{ij}$, and their numbers are:
\[
d_{x_i,y_i}\approx \delta;\; x_i,y_i\in \Delta_i:\; |\{d_{x_i ,y_i}\}|=n_{xi} n_{yi}
\]
\[
d_{x_i,y_j}\approx d_{y_i,x_j}\approx D_{ij};\; x_i,y_i\in \Delta_i;x_j,y_j\in \Delta_j:
\]
\[
|\{d_{x_i,y_j}\}|=n_{xi} n_{yj};\; |\{d_{y_i,x_j}\}|=n_{yi} n_{xj}
\]
Therefore, the proportions of the number of distances with a value of $D_{ij}$ in the ICD and BCD sets are:

For ICDs: 
\[
\frac{|\{d_{x_{ij}} \}|}{|\{d_x \}|} =\frac{2n_{xi} n_{xj}}{N_x (N_x-1)}
\]
For BCDs, considering the density ratio:
\[
\frac{|\{d_{x_i,y_j} \}|+|\{d_{y_i,x_j }\}|}{|\{d_{x,y} \}|} =\frac{\alpha n_{xi} n_{xj}+\alpha n_{xi} n_{xj}}{\alpha N_x^2 }=\frac{2n_{xi} n_{xj}}{N_x^2}
\]
The ratio of proportions of the number of distances with a value of $D_{ij}$ in the two sets is:
\[
\frac{N_x (N_x-1)}{N_x^2}=1-\frac{1}{N_x} \to 1 \; \; (N_x\to \infty)
\]
This means that the number of proportions of the number of distances with a value of $D_{ij}$ in the two sets is equal. We then examine the proportions of the number of distances with a value of $\delta$ in the ICD and BCD sets.

For ICDs:
\begin{multline*}
\sum_{i} \frac{|\{d_{x_i}\}|}{|\{d_x\}|} = \frac{\sum_{i} [n_{xi} (n_{xi}-1)]}{N_x (N_x-1)} \\ = \frac{\sum_{i} (n_{xi}^2-n_{xi} )}{N_x^2-N_x} = \frac{\sum_{i} (n_xi^2 ) -N_x}{N_x^2-N_x}
\end{multline*}
For BCDs, considering the density ratio: 
\[
\sum_{i} \frac{|\{d_{x_i,y_i } \}|}{|\{d_{x,y})\}|} = \frac{\sum_{i} (n_{xi}^2 )}{N_x^2}
\]
The ratio of proportions of the number of distances with a value of $\delta$ in the two sets is:
\begin{multline*}
\frac{\sum_{i} (n_{xi}^2 ) }{N_x^2 }\cdot \frac{N_x^2-N_x}{\sum_{i} (n_{xi}^2 ) -N_x } \\
=\sum_{i} \left(\frac{n_{xi}^2}{N_x^2} \right) \cdot \frac{1-\frac{1}{N_x}}{\sum_{i} \left(\frac{n_{xi}^2}{N_x^2} \right) -\frac{1}{N_x}}\to 1 \; \; (N_x\to \infty)
\end{multline*}
This means that the number of proportions of the number of distances with a value of $\delta$ in the two sets is equal.

In summary, the fact that the proportion of any distance value ($\delta$ or $D_{ij}$) in the ICD set for $X$ and in the BCD set for $X$ and $Y$ is equal indicates that the distributions of the ICD and BCD sets are identical, and a corresponding proof applies to the ICD set for $Y$.  \hfill $\square$
\end{proof}

\subsubsection{Necessary condition}
\textbf{Necessary condition of Theorem \ref{thm:1}.} \textit{If the distributions of the ICD and BCD sets with sufficient data points ($|X|\ \text{and} \ |Y|\to \infty$) are nearly identical, then the two classes $X$ and $Y$ must have the same distribution.}
\begin{remark}
We prove its \textbf{contrapositive}: if $X$ and $Y$ do not have the same distribution, the distributions of the ICD and BCD sets are not identical. We then apply proof by \textbf{contradiction}: suppose that $X$ and $Y$ do not have the same distribution, but the distributions of the ICD and BCD sets are identical.
\end{remark}

\begin{proof}
Suppose classes $X$ and $Y$ have the data points $N_x, N_y$, which $\frac{N_y}{N_x} =\alpha $. Divide their distribution area into many non-overlapping tiny cells (regions). In the $i$-th cell $\Delta_i$, since distributions of $X$ and $Y$ are different, according to Lemma \ref{lma:1}, the number of points in the cell $n_{xi},n_{yi}$ fulfills:
\[
\frac{n_{yi}}{n_{xi}} = \alpha _i; \; \; \exists \alpha _i \neq \alpha
\]
The scale of cells is $\delta$ and the ICDs and BCDs of the $X$ and $Y$ points in cell $\Delta_i$ are approximately $\delta$ because the cell is sufficiently small.
\[
d_{x_i}\approx d_{y_i}\approx d_{x_i,y_i}\approx \delta; \; \; x_i,y_i\in \Delta_i
\]
In the $i$-th cell $\Delta_i$:

\romannum{1}) The ICD of $X$ is $\delta$, with a proportion of:
\begin{multline} \label{eq:1}
    \sum_{i} \frac{|\{d_{x_i}\}|}{|\{d_x\}|} = \frac{\sum_{i} [n_{xi} (n_{xi}-1)]}{N_x (N_x-1)} \\
    =\frac {\sum_{i} (n_{xi}^2-n_{xi} )}{N_x^2-N_x}=\frac{\sum_{i} (n_{xi}^2 ) -N_x}{N_x^2-N_x}
\end{multline}

\romannum{2}) The ICD of $Y$ is $\delta$, with a proportion of:
\begin{multline} \label{eq:2}
    \sum_{i} \frac{|\{d_{y_i}\}|}{|\{d_y\}|} = \frac{\sum_{i} [n_{yi} (n_{yi}-1)]}{N_y (N_y-1)}=\frac {\sum_{i} (n_{yi}^2-n_{yi} )}{N_y^2-N_y}\\
    =\frac{\sum_{i} (n_{yi}^2 ) -N_y}{N_y^2-N_y}\Bigg\rvert_{\substack{N_y=\alpha N_x \\ n_{yi} = \alpha _i n_{xi}}} = \frac{\sum_{i} (\alpha _i^2 n_{xi}^2 ) -\alpha N_x}{\alpha^2 N_x^2-\alpha N_x}
\end{multline}

\romannum{3}) The BCD of $X$ and $Y$ is $\delta$, with a proportion of:
\begin{multline} \label{eq:3}
    \sum_{i} \frac{|\{d_{x_i,y_i} \}|}{|\{d_{x,y} \}|}=\frac {\sum_{i} (n_{xi} n_{yi} ) }{N_x N_y}=\frac {\sum_{i} (\alpha _i n_{xi}^2 ) }{\alpha N_x^2}
\end{multline}
For the distributions of the two sets to be identical, the ratio of proportions of the number of distances with a value of $\delta$ in the two sets must be 1, that is $\frac{(\ref{eq:3})}{(\ref{eq:1})}=\frac{(\ref{eq:3})}{(\ref{eq:2})}=1$. Therefore:

\begin{multline} \label{eq:4}
    \frac{(\ref{eq:3})}{(\ref{eq:1})}= \frac {\sum_{i} (\alpha _i n_{xi}^2 ) }{\alpha N_x^2} \cdot \frac{N_x^2-N_x}{\sum_{i} (n_{xi}^2 )-N_x}\\
    = \frac{1}{\alpha N_x^2}\sum_{i} (\alpha _i n_{xi}^2 )\cdot \frac {1-\frac {1}{N_x} }{\frac {1}{N_x^2} \sum_{i}(n_{xi}^2 ) -\frac {1}{N_x}}\Bigg\rvert_{N_x\to \infty} \\
    =\frac {1}{\alpha}\cdot \frac{\sum_{i} (\alpha_i n_{xi}^2 ) }{\sum_{i}(n_{xi}^2 )}=1
\end{multline}

Similarly,
\begin{multline} \label{eq:5}
    \frac{(\ref{eq:3})}{(\ref{eq:2})}= \frac {\sum_{i} (\alpha _i n_{xi}^2 ) }{\alpha N_x^2} \cdot \frac{\alpha^2 N_x^2-\alpha N_x}{\sum_{i} (\alpha _i^2 n_{xi}^2 ) -\alpha N_x}\\
    = \frac{\sum_{i} (\alpha _i n_{xi}^2 )}{N_x^2}\cdot \frac {\alpha-\frac {1}{N_x} }{\frac {1}{N_x^2} \sum_{i} (\alpha _i^2 n_{xi}^2 ) -\frac {\alpha}{N_x}}\Bigg\rvert_{N_x\to \infty} \\
    =\alpha \cdot \frac{\sum_{i} (\alpha_i n_{xi}^2 ) }{\sum_{i} (\alpha _i^2 n_{xi}^2 )}=1
\end{multline}

To eliminate the $\sum_{i} (\alpha _i n_{xi}^2 )$ by considering the Eq.~\ref{eq:4}~and~\ref{eq:5}, we have:
\[
\sum_{i}(n_{xi}^2 )=\frac{\sum_{i} (\alpha _i^2 n_{xi}^2 )}{\alpha^2}
\]

Let $\rho_i=\left(\frac{\alpha_i}{\alpha}\right)^2$, then,
\[
\sum_{i}(n_{xi}^2 )=\sum_{i} (\rho_i n_{xi}^2 )
\]
Since $n_{xi}$ could be any value, to hold the equation requires $\rho_i=1$. Hence:
\[
\forall \rho_i=\left(\frac{\alpha_i}{\alpha}\right)^2=1 \Rightarrow\forall \alpha_i=\alpha
\]
This contradicts $\exists \alpha_i\neq \alpha$. Therefore, the contrapositive proposition has been proved. \hfill $\square$
\end{proof}

\section{Experiments}

We test our proposed DSI measure on two-class synthetic and multi-class real datasets and compare it with other complexity measures from the Extended Complexity Library (ECoL) package \cite{4} in R. Since the DSI is computed using KS distances between the ICD and BCD sets, it ranges from 0 to 1. For separability, a higher DSI value means the dataset is easier to separate, \textit{i.e.}, it has lower data complexity. Hence, to compare it with other complexity measures, we use ($1-DSI$). In this paper, higher complexity means lower separability (\textit{i.e.} $Separability = 1- Complexity$).

\subsection{Two-class synthetic data}
\subsubsection{Typical datasets}

\begin{figure*}[h]
    \centerline{\includegraphics[width=0.85\textwidth]{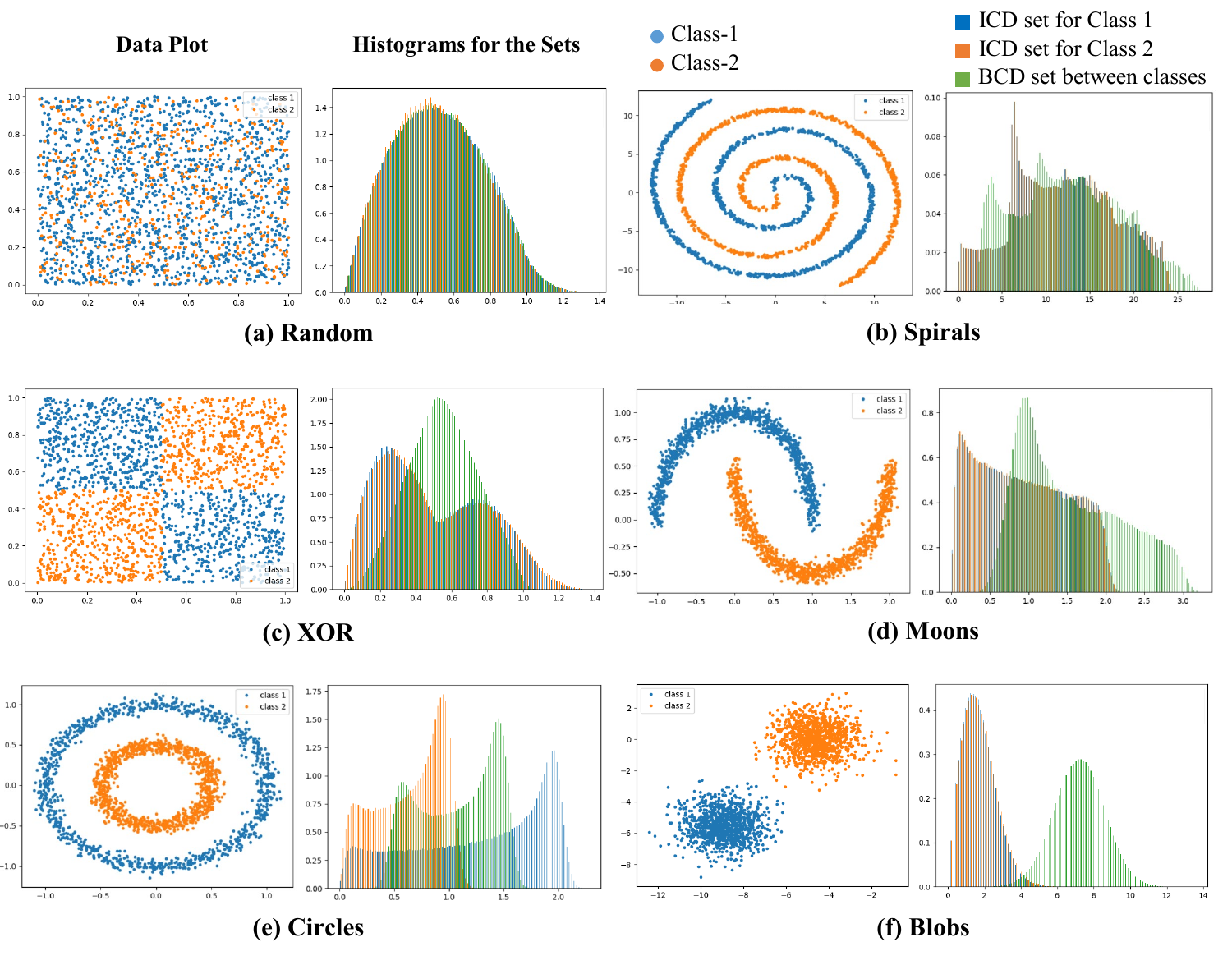}}
    \caption{Typical two-class datasets and their ICD and BCD set distributions}
    \label{fig:4}
\end{figure*}

In this section, we present the results of the DSI and the other complexity measures (listed in Table~\ref{table:1}) for several typical two-class datasets. Figure~\ref{fig:4} displays their plots and histograms of the ICD sets (for Class 1 and Class 2) and the BCD set (between Class 1 and Class 2). Each class consists of 1,000 data points.

\begin{figure}[h]
    \centerline{\includegraphics[width=0.4\textwidth]{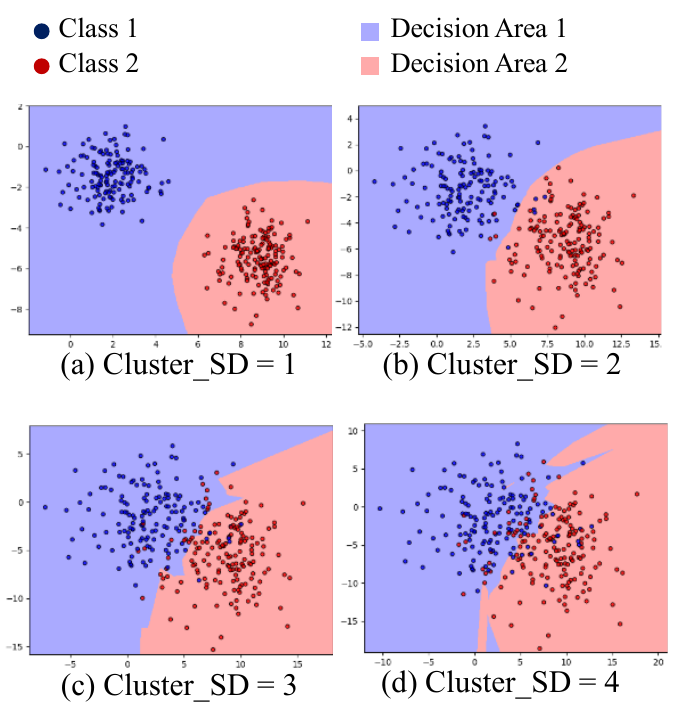}}
    \caption{Two-class datasets with different cluster standard deviation (SD) and trained decision boundaries.}
    \label{fig:clusters}
\end{figure}

Table~\ref{table:2} presents the results for these measures shown in Table~\ref{table:1} and our proposed DSI. The measures shaded in gray are considered to have failed in measuring separability and are not used for subsequent experiments. In particular, the \textit{dimensionality} and \textit{class-imbalance} measures do not work with separability in this situation. The \textit{feature-based} and \textit{linearity} measures measured the XOR dataset as having more complexity than the Random dataset; since the XOR has much clearer boundaries than Random between the two classes, these measures are inappropriate for measuring separability. N1 and N3 produce the same values for the Spiral, Moon, Circle, and Blob datasets, even though the Spiral dataset is obviously more difficult to separate than the Blob dataset, which is the most separable because a single line can be used to separate the two classes. However, the ClsCoef and Hubs measures assign the Blob dataset greater complexity than some other cases. In this experiment, N2, N4, T1, LSC, Density, and the proposed measure ($1- DSI$) are shown to accurately reflect the separability of these datasets.

\begin{table*}[t]
    \caption{Complexity measures results for the two-class datasets (Figure~\ref{fig:4}). The measures shaded in grey failed to measure separability.}
    \label{table:2}
    \centering
    \begin{tabular}{l l c c c c c c}
    \hline\noalign{\smallskip}
     \bfseries Category & \bfseries Code & \bfseries Random & \bfseries Spirals & \bfseries XOR & \bfseries Moons & \bfseries Circles & \bfseries Blobs \\ 
    \noalign{\smallskip}\hline\noalign{\smallskip}
    \rowcolor{gray!25} \cellcolor{white} &F1 & 0.998 &0.947 &1.000 &0.396 &1.000 &0.109 \\
    \rowcolor{gray!25} \cellcolor{white}&  F1v& 0.991 &0.779 &0.999 &0.110 &1.000 &0.019 \\
    \rowcolor{gray!25} \cellcolor{white} & F2& 0.996 &0.719 &0.996 &0.151 &0.329 &0.006 \\
    \rowcolor{gray!25} \cellcolor{white} &  F3& 0.997 &0.843 &0.998 &0.397 &0.708 &0.007\\
    \rowcolor{gray!25}\multirow{-5}{*}{Feature-based}\cellcolor{white} &  F4& 0.995 &0.827 &0.997 &0.199 &0.500 &0.000 \\ 
    \noalign{\smallskip}\hline\noalign{\smallskip}
    \rowcolor{gray!25} \cellcolor{white} &  L1& 0.201 &0.170 &0.328 &0.074 &0.233 &0.000 \\
    \rowcolor{gray!25} \cellcolor{white} & L2& 0.485 &0.407 &0.487 &0.114 &0.458 &0.000 \\
    \rowcolor{gray!25}\cellcolor{white}\multirow{-3}{*}{Linearity}&L3& 0.469 &0.399 &0.486 &0.055 &0.454 &0.000 \\ 
    \noalign{\smallskip}\hline\noalign{\smallskip}
    \rowcolor{gray!25} \cellcolor{white} & N1& 0.719 &0.001 &0.040 &0.001 &0.001 &0.001 \\
    & N2& 0.502 &0.052 &0.071 &0.025 &0.043 &0.017 \\
    \rowcolor{gray!25} \cellcolor{white} & N3& 0.500 &0.000 &0.019 &0.000 &0.000 &0.000 \\
    & N4& 0.450 &0.359 &0.152 &0.099 &0.162 &0.000 \\
    & T1& 0.727 &0.045 &0.043 &0.008 &0.012 &0.001 \\
    \multirow{-6}{*}{Neighborhood}&  LSC& 0.999 &0.976 &0.934 &0.840 &0.914 &0.526 \\ 
    \noalign{\smallskip}\hline\noalign{\smallskip}
    & Density& 0.916 &0.919 &0.864 &0.847 &0.880 &0.812 \\
    \rowcolor{gray!25} \cellcolor{white} & ClsCoef& 0.352 &0.343 &0.267 &0.225 &0.253 &0.332 \\
    \rowcolor{gray!25}\cellcolor{white}\multirow{-3}{*}{Network}&  Hubs& 0.775 &0.822 &0.857 &0.767 &0.650 &0.842 \\ 
    \noalign{\smallskip}\hline\noalign{\smallskip}
    \rowcolor{gray!25} \cellcolor{white} & T2& 0.001 &0.001 &0.001 &0.001 &0.001 &0.001 \\
    \rowcolor{gray!25} \cellcolor{white} & T3& 0.001 &0.001 &0.001 &0.001 &0.001 &0.001 \\
    \rowcolor{gray!25}\cellcolor{white}\multirow{-3}{*}{Dimensionality}& T4& 1.000 &1.000 &1.000 &1.000 &1.000 &1.000 \\ 
    \noalign{\smallskip}\hline\noalign{\smallskip}
    \rowcolor{gray!25} \cellcolor{white} & C1& 1.000 &1.000 &1.000 &1.000 &1.000 &1.000 \\
    \rowcolor{gray!25}\cellcolor{white}\multirow{-2}{*}{Class imbalance}& C2& 0.000 &0.000 &0.001 &0.000 &0.000 &0.000 \\ 
    \noalign{\smallskip}\hline\noalign{\smallskip}
    \textbf{Proposed}& $1-DSI$ & 0.994 &0.953 &0.775 &0.643 &0.545 &0.027\\ 
    \noalign{\smallskip}\hline
    \end{tabular}
\end{table*}

\subsubsection{Training distinctness (TD) and the two-cluster dataset} \label{sec:two-cluster}

\begin{definition} \label{def:3}
The training distinctness (TD) is the average training accuracy during the training process of a neural network classifier.
\end{definition}
\begin{remark}
To quantify the difficulty of training the classifier, we define the training distinctness (TD). A lower TD value means that a dataset is more difficult to train, and this difficulty can reflect the separability of the dataset. Hence, TD is the baseline of data separability.
\end{remark}

In this section, we synthesize a two-class dataset that has different separability levels. The dataset has two clusters, one for each class. The parameter controlling the standard deviation (SD) of clusters influences separability (Figure~\ref{fig:clusters}), and the baseline is the TD we defined.

We created nine two-class datasets using the \texttt{sklearn.datasets.make\_blobs} function in Python. Each dataset has 2,000 data points (1,000 per class) and two cluster centers for the two classes, and the SD parameters of clusters are set from 1 to 9. Along with SD of clusters increasing, distributions of two classes are more overlapped and mixed together, thus reducing the separability of the datasets. 

We use a simple fully-connected neural network (FCNN) model to classify these two-class datasets. This FCNN model has three hidden layers; there are 16, 32, and 16 neurons, respectively, with ReLU activation functions in each layer. The classifier was trained on one of the nine datasets, repeatedly from scratch. We set 1,000 epochs for each training session to compute the TD of each dataset.

In this case, separability could be clearly visualized by the complexity of the decision boundary. Figure~\ref{fig:clusters} shows that datasets with a larger cluster SD need more complex decision boundaries. In fact, if a classifier model can produce decision boundaries for any complexity, it can achieve 100\% training accuracy for any datasets (no two data points from different classes have all the same features) but the training steps (\textit{i.e.} epochs) required to reach 100\% training accuracy may vary. For a specific model, a more complex decision boundary may need more steps to train. Therefore, the average training accuracy throughout the training process – \textit{i.e.} TD – can indicate the complexity of the decision boundary and the separability of the dataset.

Since the training accuracy ranges from 0.5 to 1.0 for two-class classification, to enable a comparison with other measures that range from 0 to 1, we normalize the accuracy by the function:
\[
r\left(x\right)=(x-0.5)/0.5
\]
$rTD=r(TD)$. The range of rTD is from 0 to 1, and the lowest complexity (highest separability) is 1. We also compute N2, N4, T1, LSC, Density, and the proposed measure ($1- DSI$) for the nine datasets and present them together with rTD as a baseline for separability in Figure~\ref{fig:cluster_sd_measure}.

\begin{figure}[h]
    \centerline{\includegraphics[width=0.48\textwidth]{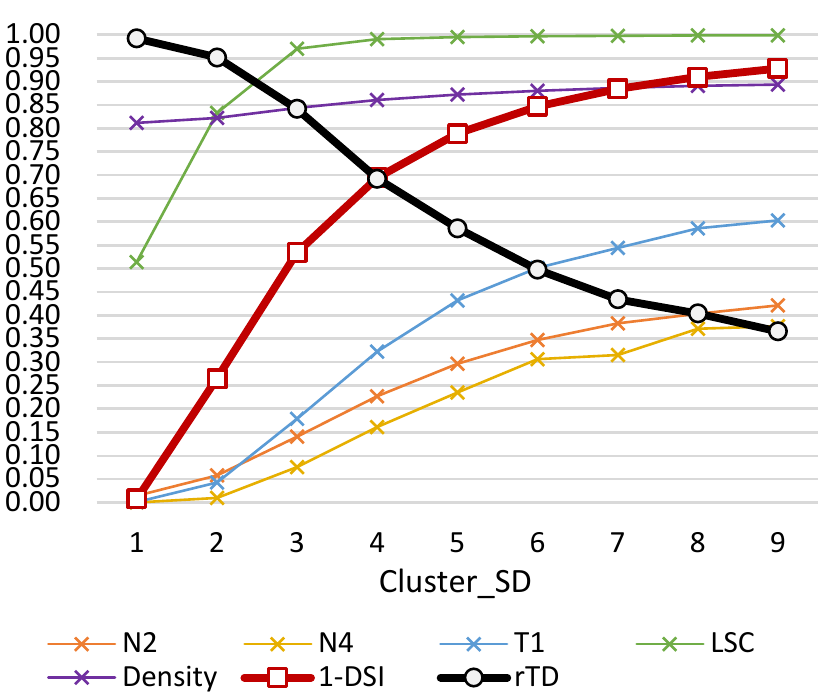}}
    \caption{Complexity measures for two-class datasets with different cluster SDs.}
    \label{fig:cluster_sd_measure}
\end{figure}

As shown in Figure~\ref{fig:cluster_sd_measure}, the rTD for datasets with larger cluster SDs is lower. Lower rTD indicates lower separability and higher complexity. The measures N2, N4, T1, and the proposed measure ($1- DSI$) reflect the complexity of these datasets well, but the LSC and Density measures do not well reflect the complexity because they have relatively high values for the linearly separable dataset (Cluster\_SD $=1$, see Figure~\ref{fig:clusters}a) and increase very slightly for the Cluster\_SD $=5$ to Cluster\_SD $=9$ datasets. The measures N2, N4, and T1 perform similarly to each other. By comparison with them, ($1- DSI$) is the most sensitive measure to the change in separability and has the widest range.

\subsection{CIFAR-10/100 datasets}
\label{cifar}
We next use real images from the CIFAR-10/100 database \cite{16} to examine the separability measures. A simple CNN with four convolutional layers, two max-pooling layers, and one dense layer is trained to classify images; \nameref{sec:appx} presents its detailed architecture.

The CNN classifier is trained on 50,000 images from the CIFAR-10/100 database. To change the classification performance (\textit{i.e.}, the TD), we apply several image pre-processing methods to the input images before training the CNN classifier. These pre-processing methods are supposed to change the distribution of the training images, and thus alter the separability of the dataset. And, the change of data separability will affect the classification results for the given CNN in terms of the TD.

\subsubsection{Subsets and the DSI}
\label{subset_cifar}
Images in the CIFAR-10 dataset are grouped into 10 classes and the CIFAR-100 dataset consists of 20 super-classes. Both CIFAR-10 and CIFAR-100 consist of 50,000 images (32x32, 8-bit RGB), and each image has 3072 pixels (features). 

\begin{figure}[h]
    \centerline{\includegraphics[width=0.48\textwidth]{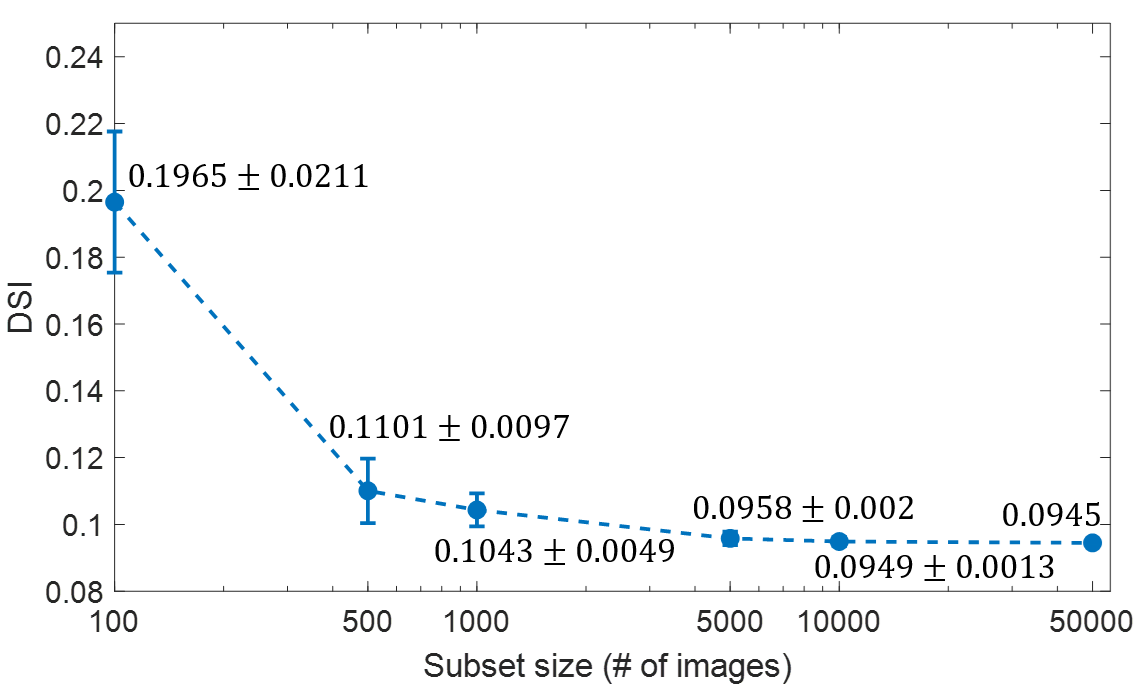}}
    \caption{DSIs of CIFAR-10 subsets}
    \label{fig:5}
\end{figure}

Since to apply the measures using all 50,000 images would be very time-consuming (including the DSI, most of the measures have a time cost of $O(n^2)$), we randomly select subsets of 1/5, 1/10, 1/50, 1/100, and 1/500 of the original training images (\textit{i.e.} without pre-processing) from CIFAR-10 and compute their DSIs. For each subset, we repeat the random selection and DSI computation eight times to calculate the mean and SD of DSIs. Figure~\ref{fig:5} shows that the subset containing 1/50 training images or more does not significantly affect the measures. For example, the DSI for the whole (50,000) training images is 0.0945, while the DSI for a subset of 1,000 randomly selected images is $0.1043 \pm 0.0049$ -- the absolute difference is up to 0.015 (16\%) but with an execution speed that is about 2,500 times greater: computing the DSI for 1,000 images requires about 30 seconds; for the whole training dataset, the DSI calculation requires about 20 hours. In addition, because the same subset is used for all measures, the comparison results are not affected. Therefore, we have randomly selected 1,000 training images to compute the measures, and this subset still accurately reflects the separability/complexity of the entire dataset.

\begin{figure*}[t]
    \centerline{\includegraphics[width=0.9\textwidth]{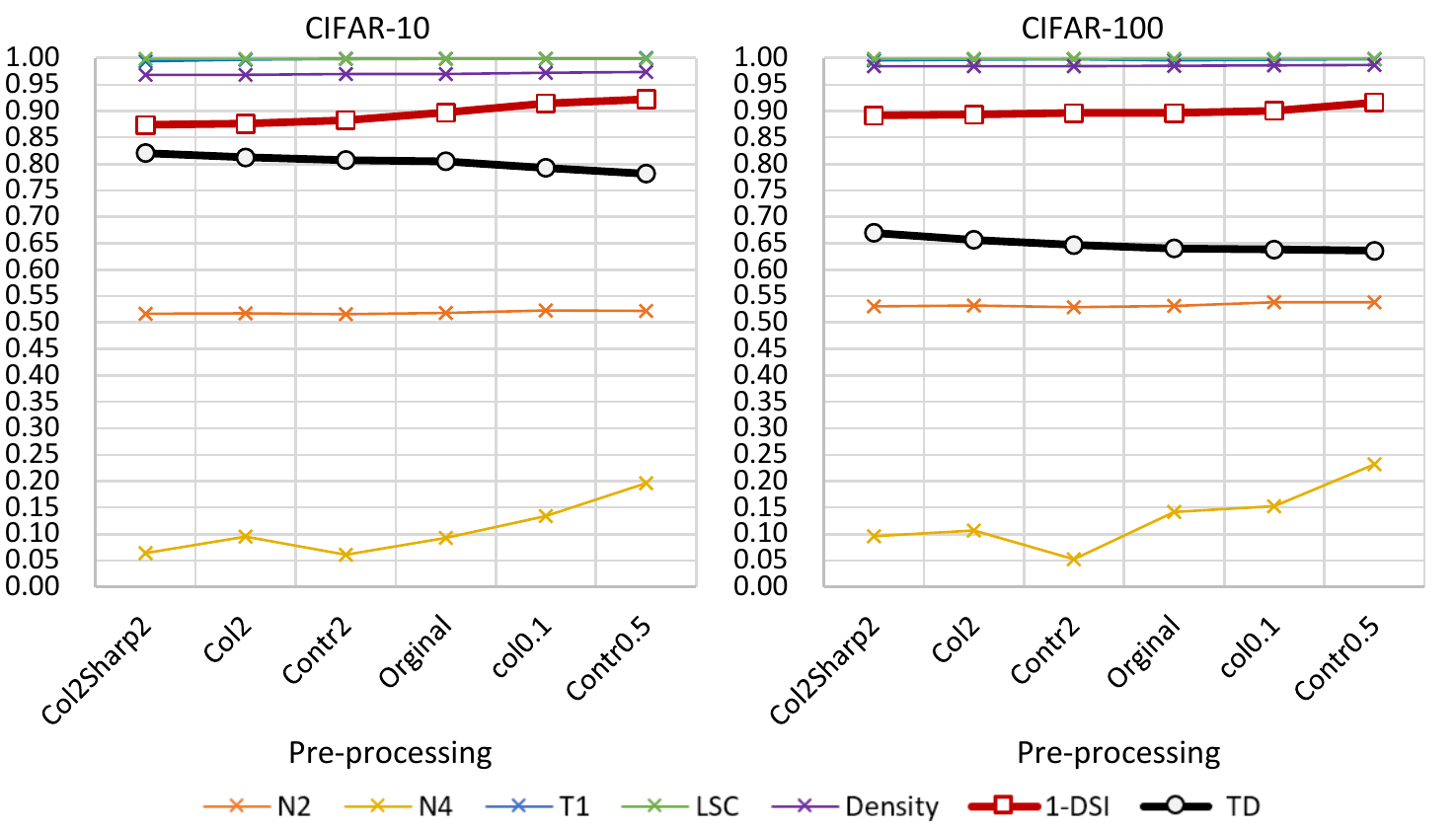}}
    \caption{Manipulation (\textit{e.g.}, pre-processing) of images in datasets can change their complexities.
    We then simultaneously compare different methods of pre-processing and of complexity measures (the y-axes) including the \textit{training distinctness} (TD) as ground truth, on the CIFAR-10/100 datasets.
    The x-axes show pre-processing methods, from left to right: Color (factor $=2$) and Sharpness (2), Color (2), Contrast (2), Color (0.1), and Contrast (0.5).}
    \label{fig:6}
\end{figure*}

\begin{table*}[h]
    \centering
    
    \caption{Values of complexity measures for CIFAR-10} \label{table:a2}
    \begin{tabular}{l l l l l l l}
    \hline\noalign{\smallskip}
    \bfseries Method Code & \bfseries Col2Sharp2 & \bfseries Col2 & \bfseries Contr2 & \bfseries Orginal & \bfseries Col0.1 & \bfseries Contr0.5 \\
    \noalign{\smallskip}\hline\noalign{\smallskip}
        N2 & 0.5169 & 0.5174 & 0.5158 & 0.5185 & 0.5226 & 0.5219 \\

    N4 & 0.0640 & 0.0950 & 0.0610 & 0.0930 & 0.1340 & 0.1960 \\

    T1 & 0.9940 & 0.9970 & 0.9990 & 0.9990 & 0.9990 & 1.0000 \\

    LSC & 0.9985 & 0.9985 & 0.9984 & 0.9985 & 0.9986 & 0.9986 \\

    Density & 0.9682 & 0.9684 & 0.9697 & 0.9701 & 0.9722 & 0.9737 \\

    $1-DSI$ & 0.8739 & 0.8762 & 0.8829 & 0.8973 & 0.9147 & 0.9224 \\

    TD & 0.8207 & 0.8122 & 0.8071 & 0.8051 & 0.7925 & 0.7818 \\
\noalign{\smallskip}\hline
    \end{tabular}
    \bigskip
    \caption{Values of complexity measures for CIFAR-100} \label{table:a3}
    \begin{tabular}{l l l l l l l}
    \hline\noalign{\smallskip}
    \bfseries Method Code & \bfseries Col2Sharp2 & \bfseries Col2 & \bfseries Contr2 & \bfseries Orginal & \bfseries Col0.1 & \bfseries Contr0.5 \\
    \noalign{\smallskip}\hline\noalign{\smallskip}
    N2 & 0.5304 & 0.5324 & 0.5292 & 0.5315 & 0.5383 & 0.5386 \\

    N4 & 0.0960 & 0.1070 & 0.0520 & 0.1420 & 0.1530 & 0.2320 \\

    T1 & 0.9960 & 0.9970 & 0.9980 & 0.9960 & 0.9970 & 0.9980 \\

    LSC & 0.9988 & 0.9988 & 0.9987 & 0.9988 & 0.9989 & 0.9988 \\

    Density & 0.9849 & 0.9849 & 0.9851 & 0.9855 & 0.9867 & 0.9873 \\

    $1-DSI$ & 0.8916 & 0.8933 & 0.8963 & 0.8964 & 0.9005 & 0.9160 \\

    TD & 0.6696 & 0.6563 & 0.6466 & 0.6403 & 0.6380 & 0.6359 \\
 \noalign{\smallskip}\hline
    \end{tabular}
\end{table*}

\subsubsection{Results}
We use the functions in \texttt{PIL.ImageEnhance} (PIL is the Python Imaging Library) with five pre-processing methods applied to the original training images from CIFAR-10/100: Color (factor $=2$) and Sharpness (2), Color (2), Contrast (2), Color (0.1), and Contrast (0.5). Including the original images, we use six image datasets to compute the DSI, TD, and other measures. For the 10-class classification, the training accuracy ranges from 0.1 to 1.0. The TD is not regularized in this section because it has a range close to $[0, 1]$.

Figure~\ref{fig:6} shows the results for CIFAR-10 and CIFAR-100. The x-axis shows the pre-processing methods applied to the datasets, decreasingly ordered from left to right by TD, which is the baseline of data separability. Since a lower TD indicates lower separability and higher complexity, the values of complexity measures should strictly increase from left to right. We put specific values of measures in the Table \ref{table:a2} and \ref{table:a3} because some differences of complexity measures' results are small and not obviously shown by curves. By examining these values, we clearly find that the measures LSC, T1 (which almost overlaps with LSC) and Density have high values and remain nearly flat from left to right (insensitive), while N2 and N4 decrease for the Contrast (2) pre-processing stage. Unlike the other measures, ($1-DSI$) monotonically increases from left to right and correctly reflects (and is more sensitive to) the complexity of these datasets. These results show the advantage of DSI and indicate that image pre-processing is useful for improving CNN performance in image classification.

\section{Discussion}

\subsection{Kolmogorov–Smirnov tests and other measures}
\label{measures}
It is noteworthy that our DSI is compatible with various measures of distances and distributions. The Euclidean distance and Kolmogorov–Smirnov (KS) distance are selected because, based on our experiments, we found that DSI has better sensitivity to separability by using those measures than by the other mentioned measures. The best sensitivity means the change of separability leads to the greatest difference of DSI. 

One key step in DSI computation is to examine the similarity of the distributions of the ICD and BCD sets. We applied the KS distance in our study. The result of a two-sample KS distance is the maximum distance between two cumulative distribution functions (CDFs):
\[
KS(P,Q)=\sup_{x} |P(x)-Q(x)|
\]
Where $P$ and $Q$ are the respective CDFs of the two distributions $p$ and $q$. 

Although many statistical measures, such as the Bhattacharyya distance, Kullback–Leibler divergence, and Jensen–Shannon divergence, could be used to compare the similarity between two distributions, most of them require the two sets to have the same number of data points. It is easy to show that the ICD and BCD sets ($|\{d_x\}|,|\{d_y\}|$, and $|\{d_{x,y}\}|$) cannot be the same size. For example, The $f$-divergence~\cite{nowozin2016f}: 
\[
D_f\left(P,Q\right)=\int q\left(x\right)f\left(\frac{p\left(x\right)}{q\left(x\right)}\right)dx
\]
cannot be used to compute the DSI because the ICD and BCD have different numbers of values, thus the distributions $p$ and $q$ are in different domains. Measures based on CDFs can solve this problem because CDFs exist in the union domain of $p$ and $q$. Therefore, the Wasserstein-distance~\cite{ramdas2017wasserstein} (W-distance) can be applied as an alternative similarity measure. For two 1-D distributions (\textit{e.g.} ICD and BCD sets), the result of W-distance represents the difference in the area of the two CDFs:
\[
W_1\left(P,Q\right)=\int\left|P\left(x\right)-Q\left(x\right)\right|dx
\]

The DSI uses the KS distance rather than the W-distance because we find that normalized W-distance is not as sensitive as the KS distance for measuring separability. To illustrate this, we compute the DSI by using the two distribution measures for the nine two-cluster datasets in Section~\ref{sec:two-cluster}. The two DSIs are then compared by the baseline rTD, which is also used in Section~\ref{sec:two-cluster}. Figure~\ref{fig:ks_w} shows that along with the separability of the datasets decreasing, KS distance has a wider range of decrease than the W-distance. Hence, the KS distance is considered a better distribution measure for the DSI in terms of revealing differences in the separability of datasets. 

\begin{figure}[h]
    \centerline{\includegraphics[width=0.48\textwidth]{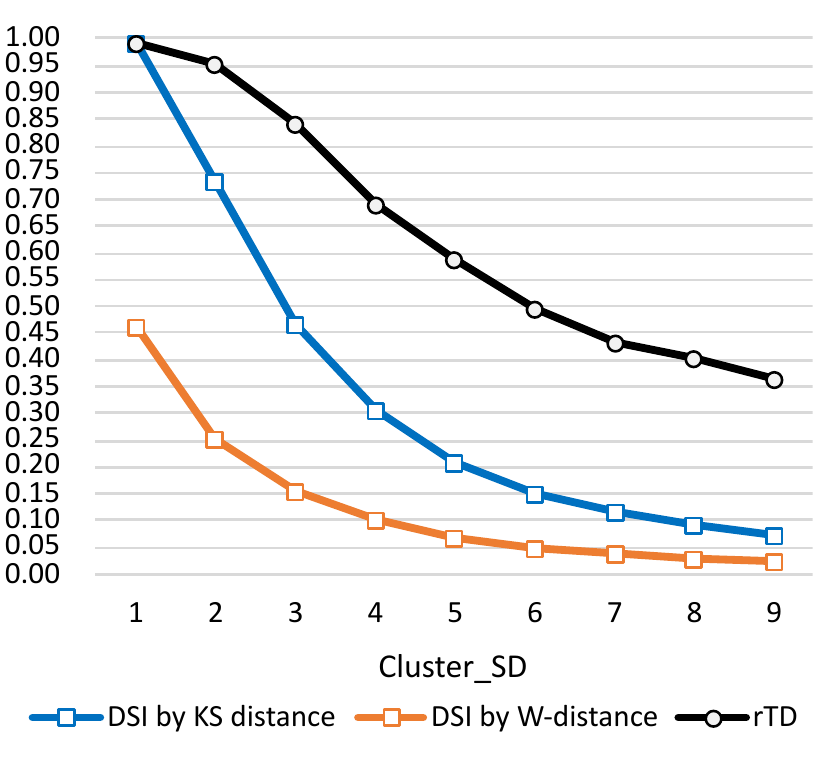}}
    \caption{DSI calculation using different distribution measures.}
    \label{fig:ks_w}
\end{figure}

\subsubsection{Distance metrics}
\label{metrics}

\begin{figure}[h]
    \centerline{\includegraphics[width=0.48\textwidth]{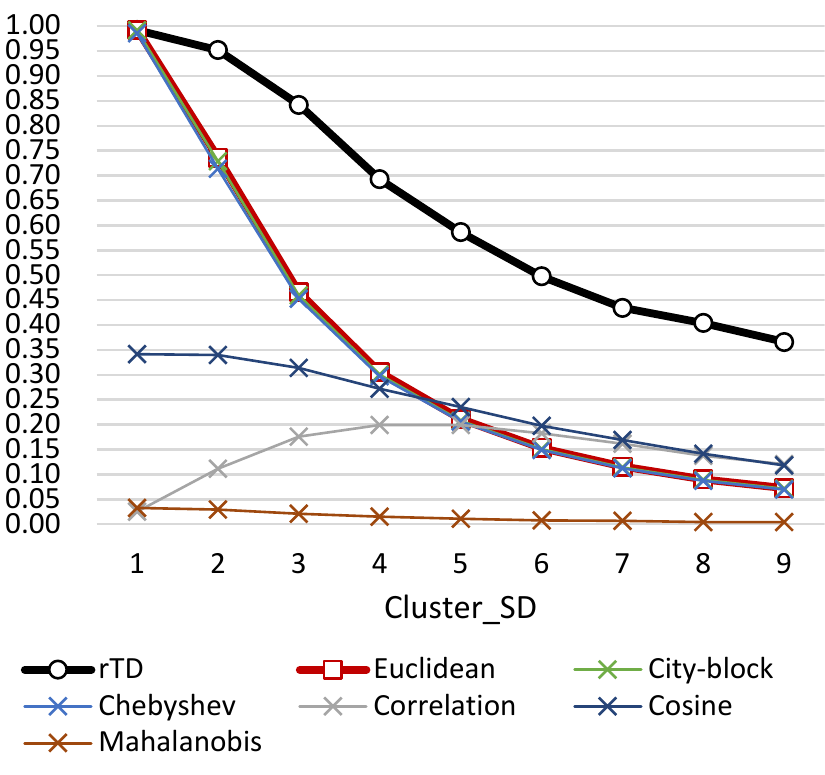}}
    \caption{DSI calculation using different distance metrics.}
    \label{fig:distance_metrics}
\end{figure}

Since the DSI examines the distributions of distances between data points, the used distance metric is another important factor. In this study, the DSI uses the Euclidean distance because it has better sensitivity to separability. We also tested several other commonly used distance metrics: City-block, Chebyshev, Correlation, Cosine, and Mahalanobis distances. We computed the DSIs based on these distance metrics using the nine two-cluster datasets in Section~\ref{sec:two-cluster} and the results are compared by the baseline rTD. Figure~\ref{fig:distance_metrics} shows that the Euclidean distance performs similarly as the City-block and Chebyshev distances. Such results indicate that the Minkowski distance metric ($p$-norm) could be suitable for the computation of DSI.

\subsection{Significance and contributions}
This work is motivated by the need for a new metric to measure the difficulty of a dataset to be classified by machine learning models. This measure of a dataset’s separability is an intrinsic characteristic of a dataset, independent of classifier models, that describes how data points belonging to different classes are mixed.
To measure the separability of mixed data points of two classes is essentially to evaluate whether the two datasets are from the same distribution. According to Theorem \ref{thm:1}, the DSI provides \textbf{an effective way to verify whether the distributions of two sample sets are identical for any dimensionality}. 

As discussed in Section~\ref{Methodological development}, if the DSI of sample sets is close to zero, the very low separability means that the two classes of data are scattered and mixed together with nearly the same distribution. The DSI transforms the comparison of distributions problem in $R^n$ (for two sample sets) to the comparison of distributions problem in $R^1$ (\textit{i.e.}, ICD and BCD sets) by computing the distances between samples. For example, in Figure~\ref{fig:4}(a), samples from Class 1 and 2 come from the same uniform distribution in $R^2$ over $[0,1)^2$. Consequently, the distributions of their ICD and BCD sets are almost identical and the DSI is about 0.0058. In this case, each class has 1,000 data points. For twice the number of data points, the DSI decreases to about 0.0030. When there are more data points of two classes from the same distribution, the DSI will approach zero, which is the limit of the DSI if the distributions of two sample sets are identical.

\subsection{Other applications and future works}
The principal use of the proposed DSI is understanding data separability, which could help in choosing a proper machine learning model for data classification \cite{13}. This will be useful to the model designer who can begin with either a small- or large-scale classifier. For example, a simpler classifier could be used for an easily-separable dataset and thus reduce both computational cost and overfitting. DSI could serve also as a way to benchmark the efficiency of a classifier, given a suitable measure of classifier complexity and computational cost.

Since DSI can evaluate whether two datasets are from the same distribution, we have applied it~\cite{guan2021novel} to evaluate generative adversarial networks (GANs) \cite{17}, competing with the existing IS and FID \cite{18} measures. As with the FID, measuring how close the distributions of real and GAN-generated images are to each other is an effective approach to assess GAN performance because the goal of GAN training is to generate new images that have the same distribution as real images. We have also applied the DSI~\cite{guan2020cluster} as an internal cluster validity index (CVI)~\cite{ARBELAITZ2013243} to evaluate clustering results because the goal of clustering is to separate a dataset into clusters, in the macro-perspective, how well a dataset has been separated could be indicated via the separability of clusters.

By examining the similarity of the two distributions, the DSI can detect (or certify) the distribution of a sample set, \textit{i.e.}, distribution estimation. Several distributions could be assumed (\textit{e.g.}, uniform or Gaussian) and a test set is created with an assumed distribution. The DSI could then be calculated using the test and sample sets. The correct assumed distribution will have a very small DSI (\textit{i.e.}, close to 0) value. 

In addition to the mentioned applications, DSI can also be used as a feature-selection method for dimensionality reduction and an anomaly detection method in data analysis. DSI has broad applications in deep learning, machine learning, and data science beyond direct quantification of separability.

\subsection{Limitations}

\begin{figure*}[t]
    \centering
    \subfloat[\centering DSI$\ \approx 0.3472$; TD$\ =0.96$ ]{{\includegraphics[width=0.42\textwidth]{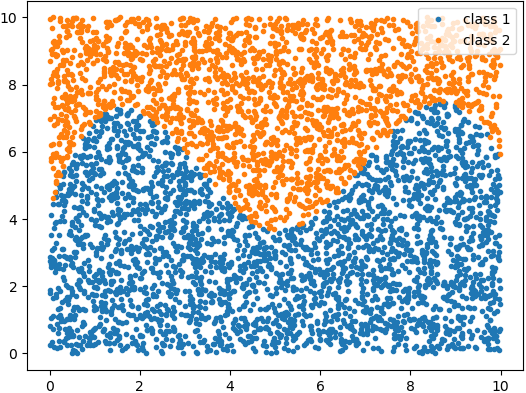} }}
    \qquad
    \subfloat[\centering DSI$\ \approx 0.3472$; TD$\ =0.79$ ]{{\includegraphics[width=0.42\textwidth]{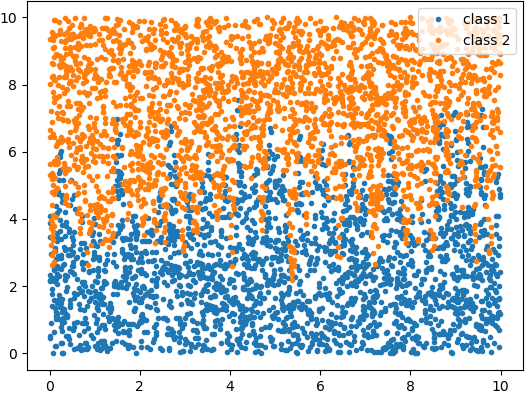} }}
    \caption{Two-class datasets with different decision boundaries. They have the same DSI but different training distinctness (TD). The dataset (b) having more complex decision boundary is more difficult to be classified.}
    \label{fig:two_boundaries}
\end{figure*}

Although DSI works for any dimensionality, dimensionality can affect data distributions and the measurement of distance, which is known as the curse of dimensionality \cite{rust1997using}, thus affecting the DSI. More studies should address the impact of dimensions on DSI and how to compare separability across different numbers of dimensions.

The separability in DSI is defined by the global distributions of data. In some cases, such separability cannot accurately reflect the complexity of the decision boundary because of local conditions. For example, in Figure~\ref{fig:two_boundaries}, two datasets have approximately the same DSI but the decision boundary complexity of the right dataset (b) is higher and its TD (by reusing the previous three-layer FCNN model in Section~\ref{sec:two-cluster}) is smaller. Therefore, the distribution-based separability cannot represent the general training/classification difficulty in terms of the complexity of decision boundaries. One of our studies addresses this problem by examining the complexity of data points on or near the decision boundary \cite{Guan2020Boundary}.

\section{Conclusion}
We proposed a novel and effective measure (\hyperref[KS]{DSI}) to verify whether the distributions of two sample sets are identical for any dimensionality. This measure has a solid theoretical basis. The core \hyperref[thm:1]{Theorem} we proved connects the distributions of two-class datasets with the distributions of their intra-class distance (\hyperref[def:1]{ICD}) sets and between-class distance (\hyperref[def:2]{BCD}) set. Usually, the datasets are in high-dimensional space and thus to compare their distributions is very difficult. By our theorem, to show that the distributions of two-class datasets are identical is equivalent to showing that the distributions of their ICD and BCD sets are identical. The distributions of sets are easy to compare because the distances are in $R^1$. The DSI is based on the KS distance between these distances’ sets.

DSI has many applications. This paper shows its core application, which is an intrinsic separability/complexity measure of a dataset. We consider that different classes of data that are mixed with the same distribution constitute the most difficult case to separate using classifiers. The DSI can indicate whether data belonging to different classes have the same distribution, and thus provides a measure of the separability of datasets. Quantification of the data separability could help a user choose a proper machine learning model for data classification without excessive iteration. Comparisons using synthetic and real datasets show that DSI performs better than many state-of-the-art separability/complexity measures and demonstrate its competitiveness as a useful measure.

Other important applications could be distribution estimation, feature selection, anomaly detection, evaluation of GANs, and playing the role of an internal cluster validity index to evaluate clustering results.

\section*{Appendix: the CNN architecture for Cifar-10/100}
\label{sec:appx}

\begin{table*}[t]
    \caption{The CNN architecture used in Section~\ref{cifar}} 
    \label{table:a1}
    \centering
    \begin{tabular}{l l}
    \hline\noalign{\smallskip}
    \bfseries Layer & \bfseries Shape \\
     \noalign{\smallskip}\hline\noalign{\smallskip}
    Input: RGB image & $32\times32\times3$ \\

    \verb|Conv_3-32 + ReLU| & $32\times32\times32$ \\

    \verb|Conv_3-32 + ReLU| & $32\times32\times32$ \\

        \verb|MaxPooling_2 + Dropout| (0.25) & $16\times16\times32$ \\

        \verb|Conv_3-64 + ReLU| & $16\times16\times64$ \\

        \verb|Conv_3-64 + ReLU| & $16\times16\times64$ \\

        \verb|MaxPooling_2 + Dropout| (0.25) & $8\times8\times64$ \\

        \verb|Flatten| & 4096 \\

        \verb|FC_512 + Dropout| (0.5) & 512 \\

        \verb|FC_10 (Cifar-10) / FC_20 (Cifar-100)| & 10 \verb|/| 20 \\

        Output (\verb|softmax|): [0,1] & 10 (Cifar-10) \verb|/| 20 (Cifar-100) \\
    \noalign{\smallskip}\hline
    \end{tabular}
\end{table*}

The CNN architecture used in Section~\ref{cifar} of the main paper consists of four convolutional layers, two max-pooling layers, and two fully connected (FC) layers. The activation function for each convolutional layer is the ReLU function, and that for output is softmax function, which maps the output value to a range of $[0,1]$, with a summation of 1. The notation \verb|Conv_3-32| indicates that there are 32 convolutional neurons (units), and the filter size in each unit is 3$\times$3 pixels (height$\times$width) in this layer. \verb|MaxPooling_2| denotes a max-pooling layer with a filter of $2\times 2$ pixels window and stride 2. In addition, \verb|FC_n| represents a FC layer with $n$ units. The dropout layer randomly sets the fraction rate of the input units to 0 for the next layer with every update during training; this layer helps the network to avoid overfitting. Table \ref{table:a1} shows the detailed architecture. Our training optimizer is RMSprop~\cite{20} with a learning rate of 1e-4 and a decay of 1e-6, the loss function is categorical cross-entropy, the updating metric is accuracy, the batch size is 32, and the number of total epochs is set at 200.


%
\section*{Compliance with ethical standards}
\textbf{Conflict of interest} The authors declare that they have no conflict of interest.

\bibliographystyle{spmpsci}      
\bibliography{ref}   


\end{document}